\documentclass{article}

\usepackage[preprint]{neurips_2023}

\usepackage[utf8]{inputenc} 
\usepackage[T1]{fontenc}    
\usepackage{hyperref}       
\usepackage{url}            
\usepackage{booktabs}       
\usepackage{amsfonts}       
\usepackage{nicefrac}       
\usepackage{microtype}      
\usepackage{xcolor}         

\usepackage{subcaption}
\usepackage{graphicx}
\usepackage{amsmath}
\usepackage{amsthm}
\DeclareMathOperator*{\argmin}{argmin} 
\DeclareMathOperator*{\proj}{Proj} 
\newtheorem{theorem}{Theorem}[]
\newtheorem{corollary}{Corollary}[]
\newtheorem{lemma}{Lemma}[]
\newtheorem{proposition}{Proposition}[]
\theoremstyle{definition}

\newtheorem{assumption}{Assumption}[]
\theoremstyle{remark}

\newcommand{\norm}[1]{\left\lVert#1\right\rVert}
\usepackage{comment}

\title{Multi-Bellman operator for convergence of \\ $Q$-learning with linear function approximation}

\author{
  Diogo S. Carvalho\thanks{diogo.s.carvalho@tecnico.ulisboa.pt}, \, Pedro A. Santos, \,and \, Francisco S. Melo 
  \\
  Instituto Superior Técnico,
  University of Lisbon,
  and INESC-ID 
}

\begin{document}

\maketitle

\begin{abstract}
%
%
%
We study the convergence of $Q$-learning with linear function approximation. Our key contribution is the introduction of a novel multi-Bellman operator that extends the traditional Bellman operator. By exploring the properties of this operator, we identify conditions under which the projected multi-Bellman operator becomes contractive, providing improved fixed-point guarantees compared to the Bellman operator.
To leverage these insights, we propose the multi $Q$-learning algorithm with linear function approximation. We demonstrate that this algorithm converges to the fixed-point of the projected multi-Bellman operator, yielding solutions of arbitrary accuracy.
Finally, we validate our approach by applying it to well-known environments, showcasing the effectiveness and applicability of our findings.
\end{abstract}

\section{Introduction}
%
Reinforcement learning aims to approximate the value of actions in different states, considering the expected sum of time-discounted rewards in a Markovian environment. The importance of this task cannot be overstated, as an accurate value function enables an agent to make optimal decisions by selecting actions with the highest value in a given state~\citep{puterman05}. Additionally, the value function facilitates environment evaluation and enables comparisons between different environments.
%

If we can use a tabular representation for the value function, meaning that we can store the value of performing each action on each state individually, the $Q$-learning algorithm converges to the correct value function~\citep{watkins1992q}. When it is not possible or desirable to store values in a table, for example when there are too many states or actions, $Q$-learning can be combined with function approximation~\citep{melo2007q}. 
%
%
Unfortunately, the combination of $Q$-learning and function approximation is troublesome. Even when the function approximation space is linear, the algorithm is not guaranteed to converge~\citep{sutton18}. In fact, the approximation problem that $Q$-learning addresses does not have, in general, a solution~\citep{melo08icml}.
%

%
To address these limitations, we propose an alternative algorithm that effectively solves the function approximation problem, unlike the original $Q$-learning. 
Both the problem and the algorithm proposed can be seen as extensions of the original.
%
%
In this work, our contributions are as follows:
\begin{itemize}
    \item Introduction of the multi-Bellman operator and analysis of its functional properties.
    \item Identification of conditions under which the projected multi-Bellman operator is contractive.
    \item Proposal of the multi $Q$-learning algorithm with linear function approximation.
    \item Theoretical and empirical demonstrations of the convergence of multi $Q$-learning.
    \item Theoretical and empirical evidence that the obtained solution can achieve arbitrary precision.
\end{itemize}
%
%


\section{Background}\label{sec:background}
%
A Markov decision problem (MDP) is a tuple $(\mathcal{X}, \mathcal{A}, \mathcal{P}, \mathcal{R}, \gamma)$, where $\mathcal{X}$ is a discrete set of states (the state space), $\mathcal{A}$ is a finite set of actions (the action space), $\mathcal{P}$ is a set of distributions over $\mathcal{X}$ for each state and action (the transitions), $\mathcal{R}$ is a set of distributions over a bounded subset of $\mathbb{R}$ for each state and action with expected value $r$ (the rewards) and $\gamma$ is a real in $[0, 1)$ (the discount).

Given an MDP, the value of a policy $\pi: \mathcal{X} \to \Delta\left(\mathcal{A}\right)$ is the function $q: \mathcal{X} \times \mathcal{A} \to \mathbb{R}$
\begin{align*}
    q(x, a) = \mathbb{E}\left[ \sum_{t = 0}^{\infty} \gamma^t r_t \mid x_0 = x, a_0 = a \right],
\end{align*}
where the expectation is with respect to rewards $r_t$ that are obtained from performing action $a_t$ on state $x_t$, states $x_{t+1}$ that are obtained by performing action $a_t$ on state $x_t$, and actions $a_t$  that are selected according to $\pi$. There is at least one policy $\pi^*$ that maximizes $q_\pi$ on every state and action~\cite[]{puterman05} and we refer to its value as $q^*$. 
The value $q^*$ satisfies the Bellman equation
\begin{align*}
    q^* = \mathbf{H}q^*,
\end{align*}
where $\mathbf{H}$ is the Bellman operator defined for arbitrary $q:\mathcal{X}\times\mathcal{A}\to\mathbb{R}$ as
\begin{align*}
    (\mathbf{H}q)(x, a) = \mathbb{E}\left[ r(x, a) + \gamma \max_{a' \in \mathcal{A}} q\left(x', a'\right) \right],
\end{align*}
where the expectation is with respect to the next state $x'$ obtained by performing action $a$ on state $x$.

Our goal is to approximate $q^*$. In other words, we want to find a good parameterized representation of the value of an optimal policy for a given MDP given a function approximation space.

\subsection{Function approximation}
Let us consider a differentiable function space ${\mathcal{H}=\{h_\omega : \mathcal{Z} \to \mathbb{R}, \omega \in \mathbb{R}^k \}}$, a distribution $\mu$ over a random $z$ in a discrete set $\mathcal{Z}$ and a loss ${l(\omega) = \frac{1}{2}\norm{ h - h_\omega }^2}$ with the $\mu$-norm $\norm{ h } = \sqrt{\mathbb{E}_\mu\left[h^2(z)\right]}$.

To approximate a function $h$ is to find an element in the subset $\proj h$ of $\mathcal{H}$ defined as
\begin{align*}
    \proj h = \argmin_{h_\omega \in \mathcal{H}} l(\omega).
\end{align*}
Any $\omega$ parameterizing a $h_\omega$ in $\proj h$ is a critical point of $l$ and must verify $\nabla_\omega l(\omega) = 0$.
\paragraph{Linear function approximation}
%
%
Given features $\phi: \mathcal{Z} \to \mathbb{R}^k$, a linear function approximation space is given by the functions such that $h_\omega(z) = \phi^T(z) \omega$. In this case, the gradient of the loss is
\begin{align*}
    \nabla_\omega l(\omega) = - \mathbb{E}\left[ \phi(z) \left( h(z) - h_\omega(z) \right) \right].
\end{align*}
Solving for $\nabla_\omega l(\omega)=0$ we obtain that
\begin{align*}
    \omega = \mathbb{E} \left[ \phi(z) \phi^T(z) \right]^{-1} \mathbb{E} \left[\phi(z)  h(z)\right].
\end{align*}
Thus, if the inverted matrix above exists, the set $\proj h$ has a single element and we can refer to both as $h_{\tilde{\omega}}$.
The solution $\tilde{\omega}$ is also the globally asymptotically stable equilibrium of the dynamical system
\begin{align*}
    \dot{\omega} = \nabla_\omega l(\omega),
\end{align*}
and the limit of the sequence of $\omega_t$ obtained by performing a discretized update with $z_t$ i.i.d. from $\mu$
\begin{align*}
    \omega_{t + 1} = \omega_t + \alpha_t  \left[\phi\left(z_t\right) \left(h(z_t) - h_{\omega_t}(z_t) \right)\right].
\end{align*}

\paragraph{Stochastic approximation}

Without access to $h$, stochastic approximation performs the update
\begin{align*}
    \omega_{t + 1} = \omega_t + \alpha_t  \left[\phi\left(z_t\right) \left( \tau_{t + 1} - h_{\omega_t}(z_t) \right)\right],
\end{align*}
with $\tau_{t+1}$ a possibly $\omega_t$-dependent estimate of ${h\left(z_t \right)}$ called the target and $\alpha_t$ a small positive real called the learning rate. 
If there exists an equilibrium for the ordinary differential equation (o.d.e.)
\begin{align*}
    \dot{\omega} = \mathbb{E}\left[\phi\left(x, a\right) \left( \tau(\omega) - h_{\omega}(z) \right)\right],
\end{align*}
where $\tau(\omega) = \mathbb{E}\left[ \tau_{t+1} \mid z_t \right]$,
and it is globally asymptotically stable, well-established conditions guarantee the stochastic approximation update converges to such equilibrium~\cite[Chapter~2]{borkar2008stochastic}.

\subsection{$Q$-learning with linear function approximation}
We consider now features $\phi:\mathcal{X}\times\mathcal{A}\to \mathbb{R}^k$, parameters $\omega$ in $\mathbb{R}^k$, a distribution over states and actions $\mu$ and linearly parameterized functions $q_\omega: \mathcal{X}\times\mathcal{A}\to\mathbb{R}$ such that $q_\omega(x, a) = \phi^T(x, a) \omega$.

In reinforcement learning we want to approximate $q^*$. Formally, we want to compute $\omega^*$ such that
\begin{align*}
    q_{\omega^*} = \proj q^*.
\end{align*}
We have no knowledge of $q^*$ and are unable to perform an exact stochastic approximation update
\begin{align*}
    \omega_{t + 1} = \omega_t + \alpha_t  \left[\phi\left(x_t, a_t\right) \left( q^*(x_t, a_t) - q_{\omega_t}(x_t, a_t)\right) \right].
\end{align*}
Thinking of the identity $q^* = \mathbf{H}q^*$, $Q$-learning performs instead the stochastic approximation update
\begin{align*}
    \omega_{t + 1} = \omega_t + \alpha_t  \left[\phi\left(x_t, a_t\right) \left( \tau_{t+1} - q_{\omega_t}(x_t, a_t) \right)\right],
\end{align*}
where the target $\tau_{t+1}$ is a sample for $\left(\mathbf{H}q_{\omega_t}\right)(x_t, a_t)$ such that
\begin{align*}
    \tau_{t+1} = r_{t} + \gamma \max_{a_{t+1} \in \mathcal{A}} q_{\omega_t}\left(x_{t+1}, a_{t + 1}\right).
\end{align*}
%
%
A solution to $Q$-learning with linear function approximation must then be the fixed-point of the projected Bellman operator, verifying the equation
\begin{align*}
    q_{\omega} = \proj \left(\mathbf{H}q_{{\omega}}\right).
\end{align*}
Unfortunately, in general such fixed-point equation does not have a solution and, even when it does, the solution may not be an asymptotically stable equilibrium of the associated dynamical system
\begin{align*}
    \dot{\omega} = \mathbb{E}\left[\phi\left(x, a\right) \left( \tau(\omega) - q_{\omega}(x, a) \right) \right].
\end{align*}
Consequently, $Q$-learning with linear function approximation can diverge.
%

The divergence of $Q$-learning is evidenced in classic counter-examples where the parameters of the approximator do not approach any solution, either oscillating within a window \citep{boyan95nips,gordon2001reinforcement} or growing without bound \citep{vanroy96ml,baird95icml}. In practice, there is also evidence of phenomena of catastrophic forgetting \citep{cahill2011catastrophic} and of convergence to incompetent solutions \citep{hasselt18arxiv}.
Currently, the theoretical results that establish convergence of $Q$-learning restrict the data or the features too much~\citep{szepesvari04icml,melo08icml}, and the proposed variants of $Q$-learning that are guaranteed to converge under more general conditions \citep{carvalho2020new,zhang2021breaking,lim2022regularized} converge to biased limit solutions that do not hold good performance guarantees~\citep{chen2022target}. Differently, our approach does not bias the solution nor restricts the data and the features too much.
%

\section{Multi-Bellman operator}\label{sec:mbo}

%
Let us define a multi-Bellman operator $\mathbf{H}^{n}$ from the space of functions $q:\mathcal{X}\times\mathcal{A}\to\mathbb{R}$ to itself such that $\mathbf{H}^{n+1}q = \mathbf{H} \left( \mathbf{H}^n q\right) $ and $\mathbf{H}^1q = \mathbf{H}q$. For example, we have that
\begin{align*}
    (\mathbf{H}^2 q)\left(x_0, a_0\right) = \mathbb{E}\left[ r\left(x_0, a_0\right) + \gamma \max_{{a_1}\in \mathcal{A}} \mathbb{E}\left[ r(x_1, a_1) + \gamma \max_{a_2 \in \mathcal{A}} q\left(x_2, a_2\right)\right] \right].
\end{align*} 
and, more generally,
\begin{align*}
    \left(\mathbf{H}^n q\right)(x_0, a_0) = \mathbb{E}\left[ r\left(x_0, a_0\right) + \gamma \max_{a_1\in \mathcal{A}} \mathbb{E}\left[ r\left(x_1, a_1\right) + \gamma \max_{a_2 \in \mathcal{A}} \mathbb{E}\left[\cdots + \gamma \max_{a_n} q\left(x_n,a_n\right)\right]\right] \right].
\end{align*}

Since ${q^*=\mathbf{H}q^*}$, also  ${q^* = \mathbf{H}^2q^*}$
and, for every, $n$ in $\mathbb{N}$ we have that
\begin{align*}
    q^* = \mathbf{H}^n q^*.
\end{align*}
First, we present the following lemma.
\begin{lemma}\label{prop:contraction-gamma-n}
    The operation $\mathbf{H}^n$ is a contraction in the $\infty$-norm with contraction factor $\gamma^n$.
\end{lemma}
While unsurprising, the result has profound implications on the use of linear function approximation.

\subsection{Linear function approximation}
%
%
%
We consider the following assumption.
\begin{assumption}
    The features are such that the covariance matrix $\mathbb{E}_\mu\left[ \phi(x, a) \phi^T(x, a)\right]$ is invertible.
\end{assumption}
Under the assumption above, we set the following result.
\begin{proposition}\label{prop:contraction:mbo}
     There exists $N \in \mathbb{N}$ such that, for all $n \geq N$, 
     %
     %
     $\proj \mathbf{H}^n$ is a contraction in the $\mu$-norm.
\end{proposition}
\begin{corollary}\label{cor:well}
    There exists $N \in \mathbb{N}$ such that, for all $n \geq N$, there exists a unique solution $\tilde{\omega}^n$ to
    \begin{align}\label{fp:pmbo}
        q_{\omega} = \proj (\mathbf{H}^n q_{\omega}).
    \end{align}
\end{corollary}
The result states that, while the projected Bellman operator may not be contractive in the $\mu$-norm, the projected multi-Bellman operator is contractive in the $\mu$-norm for sufficiently large $n$. As a consequence, whereas the fixed-point equation of the projected Bellman operator may fail to have any solution, the fixed-point equation of the projected multi-Bellman operator has a unique solution.

The previous result establishes existence and uniqueness of solution. However, it says nothing about the quality of such solution. In the following, we analyze $\tilde{\omega}^n$ in comparison with the projected $q_{\omega^*}$.
\begin{proposition}
    For all $n\geq N$, where $N$ is identified in Corollary~\ref{prop:contraction:mbo},  $\tilde{\omega}^n$ is such that
    \begin{align}
        \norm{ q^* - q_{\tilde{\omega}^n} } \leq \frac{1}{1 - \frac{\sigma_\mathrm{max}\phi_\mathrm{max}^2}{\mu_\mathrm{min}}\gamma^n}\norm{ q^* - q_{\omega^*} },
    \end{align}
where $\sigma_\mathrm{max}=\norm{\mathbb{E} \left[ \phi(x, a) \phi^T(x, a) \right]^{-1}}_2$ and $\phi_\mathrm{max}=\max_{x, a}\norm{\phi(x, a)}_2$.
\end{proposition}
\begin{corollary}\label{cor:qual}
    The sequence $\{\tilde{\omega}^n\}_{n\geq N}$ such that $N$ and $\tilde{\omega}^n$ are identified in Corollay~\ref{prop:contraction:mbo} gives
    \begin{align*}
        \lim_{n \to \infty} \norm{ \omega^* - \tilde{\omega}^n }_2 = 0.
    \end{align*}
\end{corollary}
The result states that, as $n$ increases, $\tilde{\omega}^n$ (resp. $q_{\tilde{\omega}^n}$) becomes arbitrarily close to $\omega^*$ (resp. $q_{\omega^*}$). 

%

In the following, we propose a stochastic approximation algorithm for computing $\tilde{\omega}^n$ for given $n$.

\section{Multi $Q$-learning}\label{sec:mql}
We want to solve the fixed-point equation
\begin{align*}
    \omega = \proj (\mathbf{H}^n q_{\omega}).
\end{align*}
for arbitrary $n\in\mathbb{N}$. We propose to perform the multi $Q$-learning update
\begin{align*}
    \omega_{t + 1}^n = \omega_t^n + \alpha_t \left[\phi(x_t, a_t) \left( \tau^n_{t+1} - q_{\omega_t^n}(x_t, a_t) \right)\right],
\end{align*}
where the target $\tau_{t+1}^n$ is a sample for $\left(\mathbf{H}^n q_{\omega_t^n}\right) (x_t, a_t)$ such that
\begin{align*}
    \tau_{t+1}^n = r_t + \gamma \max_{\bar{a}\in\mathcal{A}^n} \left[ r_{t + 1} + \gamma^1 r_{t + 2} + \ldots + \gamma^{n-1} q_{\omega_t^n}\left(x_{t + n}, a_{t+n} \right) \right],
\end{align*}
with $\bar{a} = \left(a_{t+1}, a_{t+2}, \ldots, a_{t+n}\right)$ and $r_{t+m}$ and $x_{t+m+1}$ are obtained by performing action $a_{t+m}$ on state~$x_{t+m}$ for $m$ between $1$ and $n$. In practice, at time step $t$, instead of building a $1$-step greedy target, multi $Q$-learning builds an $n$-step target that is obtained by trying every action on every state encountered along an $n$-step trajectory starting at $x_{t+1}$. If we were to use terminology from the planning literature, we could say the algorithm searches with fixed-depth and full-breadth, a rather unexplored setting in reinforcement learning according to \citep[Section~5.3]{moerland2023model}.

Before we present our convergence result, let us consider a couple more of assumptions.
\begin{assumption}\label{ass:data:iid}
    For all $t \in \mathbb{N}$, states and actions are i.i.d. from $\mu$ and $\mu(x, a) \geq \mu_\mathrm{min}$, with $\mu_\mathrm{min}>0$.
\end{assumption}
\begin{assumption}
    The learning rates satisfy the conditions $\sum_{t=0}^\infty \alpha_t = \infty$ and $\sum_{t=0}^\infty \alpha_t^2 < \infty$.
\end{assumption}
The existence of a solution is guaranteed by Corollary~\ref{cor:well}, whose proof resorted to Banach's fixed-point theorem and Proposition~\ref{prop:contraction:mbo}. To establish global asymptotic stability, we make use of a Lyapunov argument. 
Then, we have the following result.
\begin{theorem}\label{theo:convergence:msql}
    There exists $N$ such that, for all $n$ larger than $N$, the sequence of $\omega_t^n$ is such that
    \begin{align*}
        \lim_{t \to \infty} \norm{\tilde{\omega}^n - \omega_t^n }_2 = 0.
    \end{align*}
\end{theorem}
%
%

The result establishes conditions under which multi $Q$-learning converges to the unique solution of the fixed-point equation of the projected multi-Bellman operator~\eqref{fp:pmbo}. 
The assumptions are commonplace~\citep{vanroy96ml,carvalho2020new,melo08icml}. Assumptions 1 and 3 are mild. Assumption 2 is the most restrictive of the three, as it requires the data distribution to have no shift during training. In practice, that is usually not the case. For example it is common that throughout the interaction of the agent with the environment, the policy used to collect data changes in response to changes in the approximated value function. Consequently, in general, the samples are not i.i.d. to a fixed distribution. Nevertheless, we can use replay buffers to slow down the data distribution shift. In the limit, the case of offline reinforcement learning, where the replay buffer is completed before reinforcement learning starts, the distribution is indeed fixed and samples are independent and identically distributed. Nevertheless, we do believe our result still holds if the data distribution, while changing in response to the reinforcement learning, converges. 

%
%
%
%

%
For performing the maximization in the target for the update of multi $Q$-learning, we require the ability to simulate transitions and rewards.
%
%
For our experiments, we assume access to a simulator. We highlight the limitation in Section~\ref{ss:limitations}. Nevertheless, we highlight that it is possible to use a learned model of the transitions and rewards. Such use would not harm the convergence guarantees of multi $Q$-learning. In fact, it is even possible to use non-parametric models, such as replay buffers, to sample transitions and rewards, without prejudice of the convergence established. It is also possible to concurrently learn parametric models in a supervised way, even though the theoretical analysis for the extension to this model-learning interaction would require additional mathematical machinery. %
%

We finish the section with a sketch of the proof of Theorem~\ref{theo:convergence:msql}, referring to the appendix for details.
\begin{proof}
Under assumptions 1, 2 and 3, respectively on the data distribution, the features and the learning rates, multi $Q$-learning satisfies conditions under which a stochastic approximation algorithm converges to the equilibrium of the associated dynamical system described in Section~\ref{sec:background}. 

We consider the result of~\cite{borkar2008stochastic} that we reproduce in the supplementary material as Theorem~\ref{theo:borkar:sa}.
Therein, we identify four conditions that we need to prove that hold for multi $Q$-learning. 

First, the expected update must be a smooth function of the parameters $\omega$. To verify this condition, we consider, for some $n\in\mathbb{N}$, the expected update map $g: \mathbb{R}^k \to \mathbb{R}^k$ such that 
\begin{align*}
    g(\omega) = \mathbb{E}\left[ \phi(x, a) \left( \tau^n(\omega) - q_\omega(x, a) \right) \right]
\end{align*}
and show that it is a Lipschitz function of the parameters in Lemma~\ref{lemma:lip}.

Second, the expected difference between the expected update and the actual update, i.e., the noise, must equal zero when conditioned on the past, have bounded expectation and bounded variance. We consider the noise sequence $\{m_t\}_{t\in\mathbb{N}}$ such that
\begin{align*}
    m_{t+1} = \phi(x_t, a_t)\left(\tau^n_{t+1} - q_{\omega_t}(x_t, a_t)\right) - g(\omega_t)
\end{align*}
and verify it forms a martingale difference sequence with bounded variance in Lemma~\ref{lemma:martingale}. 

Third, in Lemma~\ref{lemma:ode}, we show that, for sufficiently large $N$ and $n\geq N$, the o.d.e.
\begin{align*}
    \dot{\omega} = g(\omega)
\end{align*}
has a unique and globally asymptotically stable equilibrium $\tilde{\omega}^n$ that solves the fixed-point equation
\begin{align*}
    \omega = \mathbb{E}\left[ \phi\left(x, a\right) \phi^T\left( x, a \right) \right]^{-1} \mathbb{E}\left[ \phi\left(x, a\right)\tau^n(\omega)  \right].
\end{align*}
To establish the existence and uniqueness of a solution we use Corollary~\ref{cor:well}, which in turn uses Banach's fixed-point theorem and Proposition~\ref{prop:contraction:mbo}. To establish global asymptotic stability, we make a Lyapunov argument. 
We establish the result for all $n \geq N$ with $N = - \log_\gamma \left( \frac{\sigma_\mathrm{max} \phi_\mathrm{max}^2}{\mu_\mathrm{min}} \right)$. 
We do not say, however, that this is the minimum $N$ such that result would hold. Specifically, we do not guarantee our convergence result is the tightest possible.

The first three conditions ensure that, if the updates remain bounded, they converge to $\tilde{\omega}^n$~\cite[Chapter~2]{borkar2008stochastic}. The fourth condition, that we verify in Lemma~\ref{lemma:bound}, ensures such boundedness. 

Having verified the conditions of Theorem~\ref{theo:borkar:sa}, we are able to conclude that, for sufficiently large $N$, the sequence of $\omega_t^n$ generated by multi $Q$-learning converges to $\tilde{\omega}^n$ with probability 1.
\end{proof}

\section{Experiments}\label{sec:experiments}
First, we evaluate multi $Q$-learning for growing $n$ first in the task of approximating $q^*$. We use the classic counter-examples for the convergence of $Q$-learning with linear function approximation.
Then, we consider the task of using a learned approximation of $q^*$ to select actions. We use the classic control environments. For the function approximation space, we use discretized Gaussian features. We use an $\epsilon$-greedy policy where $\epsilon$ decays linearly from 100\% to 5\% during the first half of interactions and remains constant afterwards. We use a replay buffer with 20\% of the total number of timesteps used for the environment. We further detail hyperparameters used for each environment in the corresponding paragraph. We average the results across five runs, with standard deviation intervals, and plot a moving average of the last $5\%$ of the total number of time steps.

\subsection{Classic counter-examples}\label{sec:evaluation:prediction}
\paragraph{$\omega \to 2 \omega$}
The $\omega \to 2\omega$ classic counter-example is due to \cite{vanroy96ml}. Here, the MDP has two states $y_1$ and $y_2$ and only one action $b_1$. Performing the action on any of the two states always takes the agent to the second state. The reward received is always zero. Therefore, $q^*$ is zero.
To approximate $q^*$ we consider features $\phi: \mathcal{X} \times \mathcal{A} \to \mathbb{R}$ such that $\phi(y_1, b_1) = 1$ and $\phi(y_2, b_1) = 2$. 
The projection of $q^*$ is $\omega^* = 0$ and $q_{\omega^*}$ equals the solution $q^*$. $\omega^*$ also verifies the Bellman fixed-point equation $q_{\omega^*} = \proj \left(\mathbf{H} q_{\omega^*}\right)$. Regardless, considering a discount factor of $0.9$ and a learning rate of $10^{-2}$, when the distribution over states and actions is uniform, the parameters of $Q$-learning, that is multi $Q$-learning with $n=1$, diverge to infinity. 
Figure \ref{fig:evaluation:w2w} shows that for sufficiently large $n$, specifically $n=4$, we have convergence to the correct solution.

\paragraph{Star}
The Star classic counter-example is due to \cite{baird95icml}. Here, the MDP has six states $y_1$ to $y_6$ and two actions $b_1$ and $b_2$. The first action always takes the agent two the last state, the second action takes the agent to any of the first five states uniformly. The reward received is always zero. Therefore, $q^*$ is zero.
To approximate $q^*$ we consider features $\phi: \mathcal{X} \times \mathcal{A} \to \mathbb{R}^{13}$ such that, for $j$  between $1$ and $6$, for all $i$ between $1$ and $13$ $\phi_i(y_j, b_2)= \mathbf{1}(i = j + 1)$, for $i$ between $2$ and $6$ $\phi_i(y_j, b_1) = 2 \cdot \mathbf{1}(i = j+1)$, for $j$ between $1$ and $5$ $\phi_1(y_j, b_1) = \mathbf{1}(j \leq 5)$ and $\phi_1(y_6, b_1) = 2$, $\phi_i(y_j, b_1) = 0$ otherwise.
The projection of $q^*$ is $\omega^* = 0$ and $q_{\omega^*}$ equals the solution $q^*$. $\omega^*$ also verifies the Bellman fixed-point equation $q_{\omega^*} = \proj \left(\mathbf{H} q_{\omega^*}\right)$. Despite the apparently benign conditions, considering a discount factor of $0.995$ and a learning rate of $10^{-2}$, when the distribution over states and action is generated by selecting the first action one sixth of the times and the second action five sixths of the times, the parameters of $Q$-learning, that is multi $Q$-learning with $n=1$, diverge to infinity. 
Figure \ref{fig:evaluation:star} shows that for sufficiently large $n$, specifically $n=4$, we have convergence to the correct identically zero solution.

\begin{figure}[t]
    \centering
    \begin{subfigure}[b]{0.48\textwidth}
        \centering
        \includegraphics[height=0.75\textwidth]{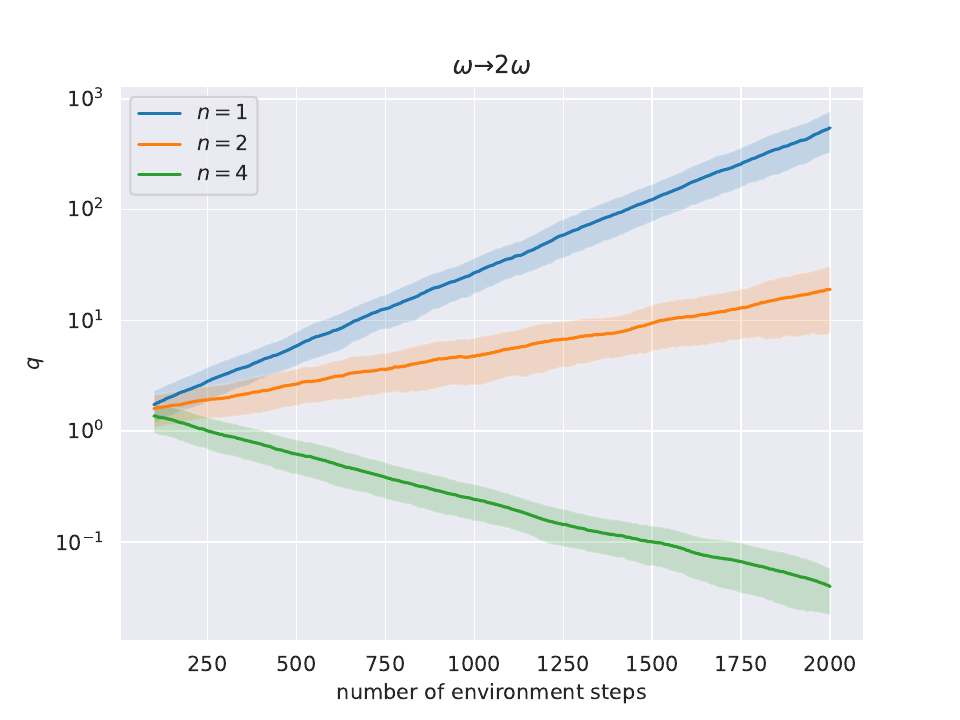}
        \caption{$\omega \to 2 \omega$.}
        \label{fig:evaluation:w2w}
    \end{subfigure}
    \hfill
    \begin{subfigure}[b]{0.48\textwidth}
        \centering
        \includegraphics[height=0.75\textwidth]{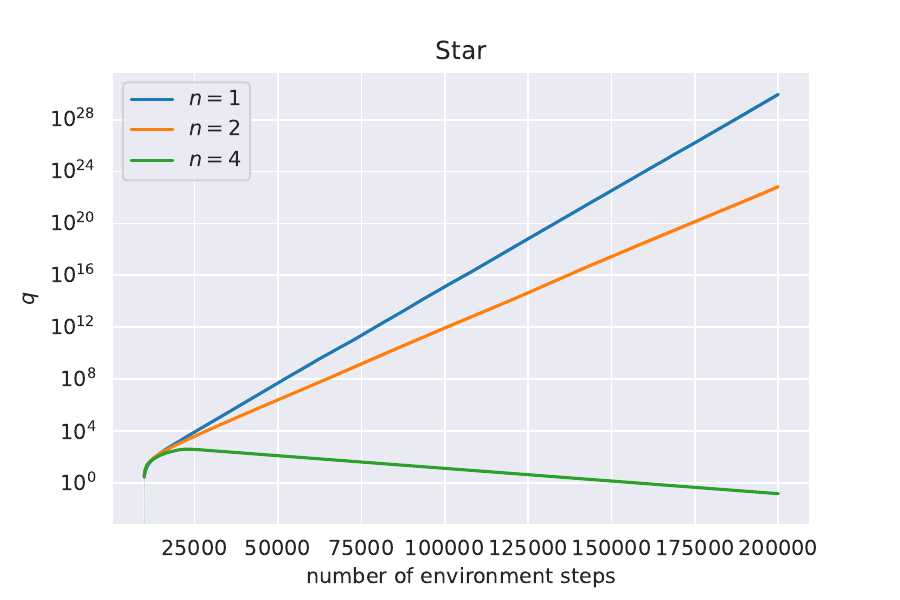}
        \caption{Star.}
        \label{fig:evaluation:star}
    \end{subfigure}
    \caption{Classic counter-examples. The $x$-axis shows the number of environment steps. The $y$-axis shows the estimate of $q$-values of the state visited. In both environments, $n=4$ was necessary and sufficient to achieve convergence to the exact optimal $q$-values which are identically zero.}
    \label{fig:evaluation:classic-counters}
\end{figure}

\subsection{Classic control}
\paragraph{Cartpole}
Cartpole is a classic control problem proposed by~\cite{barto1983neuronlike}, where a cart balances a pole. The actions of the agent are to push the cart left or right. The state space is a four-tuple with the position and velocity of the cart and the angle and angular velocity of the pole. The reward is always one unless the pole falls and the agent reaches a terminal state with reward zero. During training, episodes last at most five hundred interactions. We use Gaussian features in~$\mathbb{R}^{16}$, obtained by discretizing each dimension of the state space in two, a discount factor of $0.99$ and a learning rate of $3\cdot10^{-2}$. Figure~\ref{fig:evaluation:ablation:cartpole} shows the results. $Q$-learning can not perform. When $n \geq 2$, multi $Q$-learning is able to balance the pole and collect rewards.

\paragraph{Mountaincar}
Mountaincar is a classic control problem proposed by~\cite{Moore90efficientmemory-based} where a car must climb a valley. The actions of the agent are to push the car left, push the car right or do nothing. The state space is a double of position and velocity of the car. The reward is always minus one unless the car is at the top of the valley, after which the agent reaches a terminal state with reward zero. During training, episodes last a maximum of two hundred interactions. We use Gaussian features in~$\mathbb{R}^{256}$ obtained by discretizing each dimension of the state space in sixteen, a discount factor of $0.99$ and a learning rate of $3\cdot10^{-3}$. Figure~\ref{fig:evaluation:ablation:mountaincar} shows the results. For all $n$, multi $Q$-learning is able to select actions to successfully climb the hill and solve the problem.

\paragraph{Acrobot}
Acrobot is a classic control problem proposed by~\cite{sutton1995generalization} where a joint actuates two links such that one end is fixed and the other is free. The actions of the agent are to apply a negative torque to the joint, apply a positive torque to the joint or do nothing. The state space is the six-tuple composed of sine, cosine and angular velocity of each link. The reward is always minus one unless the free end of the links reaches a target height, after which the agent reaches a terminal state with reward zero. During training, each episode last a maximum number of five hundred interactions. We use Gaussian features in~$\mathbb{R}^{4096}$ obtained by discretizing each dimension of the state space in four, a discount factor of $0.99$ and a learning rate of $3\cdot10^{-3}$. Figure~\ref{fig:evaluation:acrobot} shows the results. As $n$ increases, the performance of multi $Q$-learning increases and becomes more stable.

\begin{figure}[t]
    \centering
    \begin{subfigure}[b]{0.32\textwidth}
        \centering
        \includegraphics[height=0.75\textwidth]{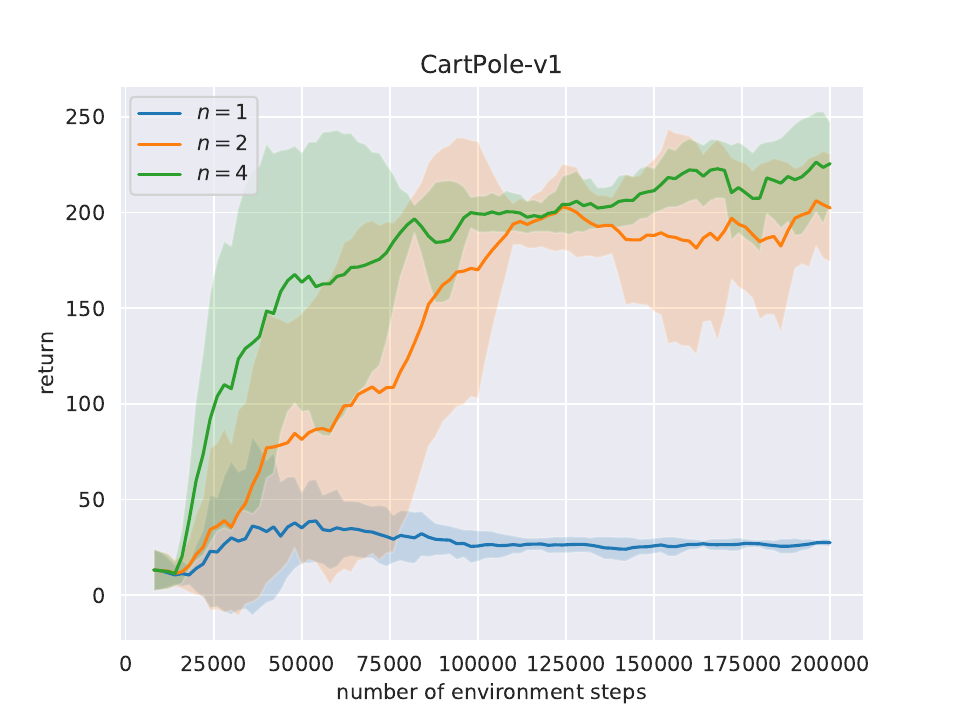}
        \caption{Cartpole.}
        \label{fig:evaluation:ablation:cartpole}
    \end{subfigure}
    \hfill
    \begin{subfigure}[b]{0.32\textwidth}
        \centering
        \includegraphics[height=0.75\textwidth]{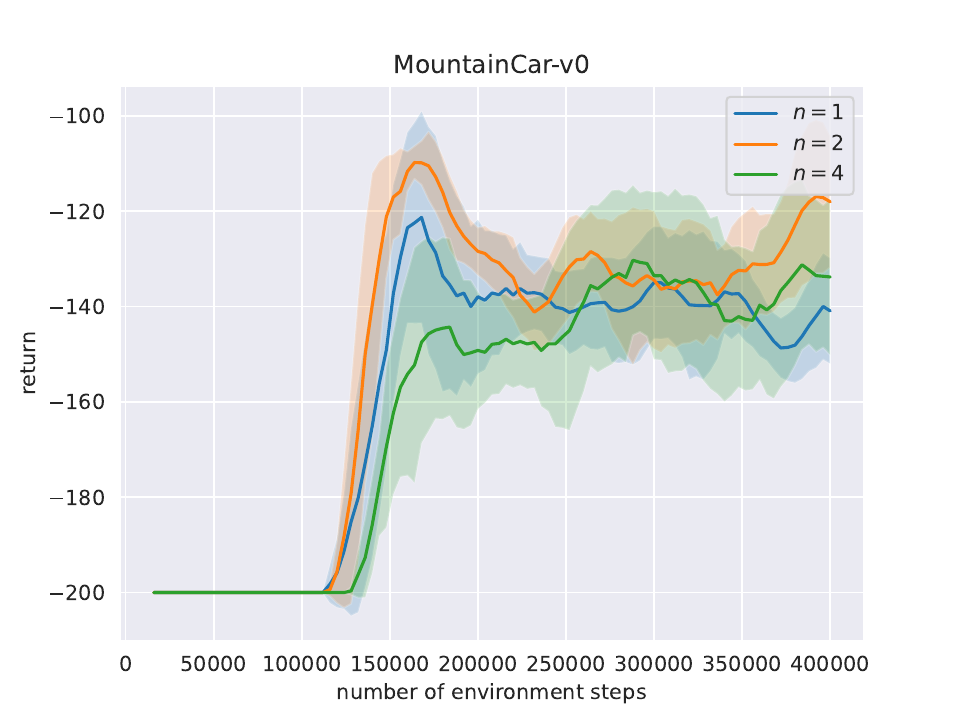}
        \caption{Mountaincar.}
        \label{fig:evaluation:ablation:mountaincar}
    \end{subfigure}
    \hfill
    \begin{subfigure}[b]{0.32\textwidth}
        \centering
        \includegraphics[height=0.75\textwidth]{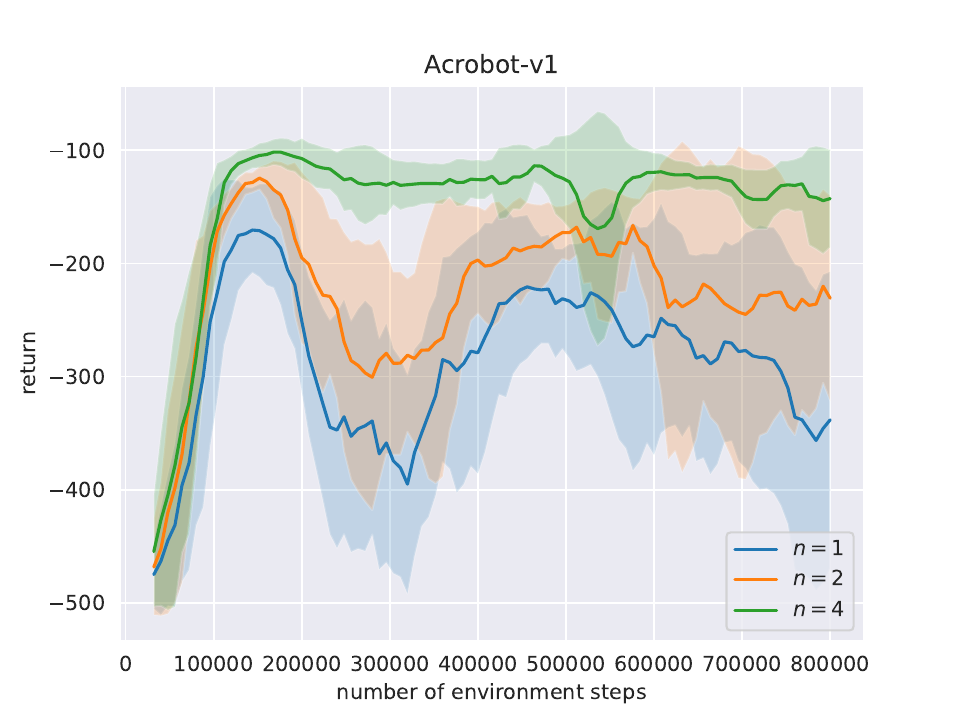}      
        \caption{Acrobot.}
        \label{fig:evaluation:acrobot}
    \end{subfigure}
    \caption{Classic control problems. The $x$-axis shows the number of environment steps. The $y$-axis shows the average return achieved over 30 evaluation episodes. We can see that, as $n$ increases, the performance of multi $Q$-learning with linear function approximation remains or improves.}
    \label{fig:evaluation:classic-control}
\end{figure}

\section{Related work}\label{sec:rw}
\subsection{Convergence results}

We analyze results on the convergence of $Q$-learning with linear function approximation. A few works provide conditions under which the standard $Q$-learning method converges. Others, propose a different learning objective and thus result in significantly modified variants of $Q$-learning, which are the gradient-TD methods. More recently, other works exploit the finding that regularization can ensure convergence of the $Q$-learning updates. From a different perspective, a few works have analyzed the impact of lookahead in policy improvement and evalutation steps.

\paragraph{Convergence conditions}
Around the time the divergence of $Q$-learning was established, some works aimed at reaching conditions under which $Q$-learning or small variants of it would converge.
The first works proposed restrictions of the function approximation spaces themselves. They establish that with some specific choices of features, $Q$-learning is guaranteed to converge~\citep{singh1994reinforcement,ormoneit2002kernel,szepesvari2004interpolation}. 
The linear architectures considered across the three works are extensions of what would be the one-hot representation in the tabular case.
Afterwards, another work considered again general linear function approximation architectures. \citet{melo08icml} prove that $Q$-learning with linear function approximation converges if the distribution of state-action pairs the agent uses to learn is sufficiently close to the distribution that the optimal policy would induce. Such case restricts the convergence of $Q$-learning closer to on-policy settings.

\paragraph{Gradient-TD methods}\label{sec:gtdm}
Instead of finding conditions under which $Q$-learning converges, a different line of works proposes to take a step back and modify the objective that $Q$-learning is trying to solve.
\citet{maei2010toward} propose to perform full-gradient descent on a different objective, called the projected Bellman error. The resulting algorithm, called Greedy-GQ, is part of a class of gradient-TD methods. 
Being a full gradient method, Greedy-GQ is provably convergent to a minimum in the linear approximation space. 
However, the method can converge to strictly local minima and there is no guarantee that the resulting greedy policy is a good control policy \citep{scherrer2010should}. Gradient-TD methods are also less efficient than semi-gradient methods~\citep{mahadevan2014proximal,du2017stochastic}.

\paragraph{Regularized methods}
The problem of divergence of $Q$-learning with function approximation was significantly revived after an empirical success story of $Q$-learning with deep neural networks~\citep{mnih2015human}. One of the components of the renowned deep $Q$-network (DQN) therein is a target network that aims at compensating the instability generated by the $Q$-learning updates.
While the DQN is not provably convergent, its empirical success inspired theoretical results. The works of \citet{carvalho2020new}, \citet{zhang2021breaking} and \citet{chen2022target} provided convergence guarantees for variants with target networks. Additionally, another work points out that the target network can be seen as a regularizer~\cite{piche2021beyond}.
As hinted by \citet{farahmand2011regularization}, several works prove that various forms of regularization of the $Q$-values or the parameters themselves can stabilize $Q$-learning, resulting in a convergent algorithms \citep{zhang2021breaking,lim2022regularized,agarwal2021online}. However, the introduction of any of the regularizers biases the limit solution encountered.
%




\paragraph{Non-linear function approximation}

The behaviour of $Q$-learning with linear function approximation has been the focus of several theoretical works. In practice, however, $Q$-learning is mostly used with non-linear function approximation, especially through neural networks \citep{mnih2015human}. Still, there are works that address this more general setting.
%
%
A recent work suggests a loss function that is decreasing over time, assuming the neural network converges to a target network at each step \citep{wang2021convergent}. However, having a loss function that is monotonically decreasing does not imply that neither the parameters of the approximator are converging nor that the $Q$-values are converging.
\citet{NEURIPS2019_98baeb82} and \citet{xu2020finite} provide finite-time analysis of $Q$-learning with over-parameterized neural networks that imply that as the size of the network grows to infinity, convergence is guaranteed at a sub-linear rate. We note that as the size of the network grows to infinity, the learning architecture also grows closer to a tabular representation. Therefore, despite interesting, the results do not imply convergence when a practical neural network is employed.

\subsection{Lookahead, planning and model-based reinforcement learning}
The multi-Bellman operator considered relates to other operators that lookahead. The multi $Q$-learning algorithm relates to other planning and model-based reinforcement learning algorithms. 

\paragraph{Lookahead}
In the context of policy iteration, several works hinted at the theoretical benefits of lookahead in the tabular case \cite{de2018multi,efroni2018beyond}. \cite{efroni2018multiple} then identified a problem with soft policy improvement for lookahead policies, which happens if function approximation is used. Specifically, contrarily to what happens with a single-step policy improvement step, a multistep policy improvement step is not necessarily monotonically increasing. 
However, \cite{efroni2020online} and \cite{winnicki2022reinforcement} respectively provide finite-time and asymptotic results for lookahead in approximate policy iteration. Our work differs both on the problem side and on the solution side. On the problem side, we focus on the problem of divergence of $Q$-learning, a value-based algorithm, when used with linear function approximation architectures that are not tables. The projected multi-Bellman operator that we introduced differs from ones in the referred works in that it is designed for evaluation of the optimal policy, not evaluating or improving on a given policy.

\paragraph{Planning and model-based reinforcement learning}
Multi $Q$-learning can be casted under the umbrella of multi-step real time dynamic programming algorithms, as defined by \cite{efroni2018multiple} and discussed in \cite{moerland2020think}. Such algorithms integrate planning and learning and have several successful practical applications \citep{silver2017mastering,silver2018general}. We refer to the survey of \cite{moerland2023model} for varied interesting algorithms that include $Q(\sigma)$~\citep{de2018multi}, tree-backup~\citep{precup2000eligibility} and multi-step expected SARSA~\citep{sutton18}. While most such algorithms are policy-based and work on-policy, requiring separate value and policy networks and behavior-dependent solutions, in our case, we have an off-policy, value-based algorithm. Moreover, multi $Q$-learning plans exclusively at training time---planning is not necessary in order to select an action to execute. Multi $Q$-learning plans with fixed depth and full breadth, which is a rather unexplored setting according to \citep{moerland2023model}. Specifically, the depth is not adaptive nor full and every action is tried on every state along a planning tree. Unfortunately, we do not believe adaptive depth or limited breadth would result in convergence. The setting could, however, bring computational benefits that could make multi $Q$-learning more practical.

\section{Conclusion}
%
In conclusion, our work has made significant contributions to addressing the convergence challenge in $Q$-learning with linear function approximation. By introducing the multi-Bellman operator and demonstrating its contractive nature, we have paved the way for improved convergence properties. The proposed multi $Q$-learning algorithm effectively approximates the fixed-point solution of the projected multi-Bellman operator. Importantly, we have shown that the algorithm converges under relatively mild conditions.
%
%
The implications of our findings extend beyond this study, as they represent a substantial breakthrough in the problem of convergence in $Q$-learning with linear function approximation. Our work has the potential to inspire further advancements in theory and algorithm development within the field of reinforcement learning research.

We highlight two limitations of our work below. Afterwards, we link them with future work.

\subsection{Limitations}\label{ss:limitations}
%
%
%

Even though our algorithm can be combined with models that are learned concurrently to the reinforcement learning update, in our experiments, we used a known model of the environment. It is not the case that such known model is always available in applications. 

%
Besides, multi $Q$-learning plans by performing every action on every reached state along trajectories of $n$ steps. Thus, the computational cost of performing an update grows exponentially with $n$, where the base for the exponent is the size of the action space.

\subsection{Future work}\label{ss:fw}
%

%
There are several techniques for learning a model of the environment. For example, we can learn the model previously or concurrently to reinforcement learning; the model can be exact or inexact; the model can be parametric or non-parametric. Multi $Q$-learning can be combined with any of the mentioned model-based approaches in order to remove the need for a known model.

%
Multi $Q$-learning does not hold so benign properties when we are unable to perform every action on every state along $n$ step trajectories. In case only some actions are performed on each state, we cannot show convergence to the solution to an appropriate fixed-point equation. Nevertheless, we expect that, in some cases, the computational benefits of not performing every action on every state to outweigh the theoretical comfort. A practical analysis of this trade-off would be valuable.

%
%
It would also be interesting to analyze multi $Q$-learning with non-linear function approximation, theoretically and empirically. Specifically, under which conditions and function approximation architectures would our convergence result still hold? How would multi $Q$-learning with non-linear function approximation compare, in practice, with relevant policy-free reinforcement learning algorithms such as the DQN~\citep{mnih2015human}?


\bibliography{biblio}


\clearpage
\appendix
\section{Appendix}
\setcounter{subsection}{1}
In this appendix, we provide proof of all theoretical results presented in the main document. In Section~\ref{sec:app:back}, we reproduce a general convergence result for stochastic approximation that was referred to in Section~\ref{sec:background} of the main document. In Section~\ref{sec:app:mbo}, we provide proof for the theoretical results we presented in Section~\ref{sec:mbo}, referring to properties of the multi-Bellman operator proposed. Finally, in Section~\ref{sec:app:mql}, we prove the convergence result for multi $Q$-learning that appeared in Section~\ref{sec:mql} of the main document. The latter proof combined the general stochastic approximation convergence result of Section~\ref{sec:app:back} and the properties of the multi-Bellman operator of Section~\ref{sec:app:mbo}.

\subsection{Background}\label{sec:app:back}
In this section, we reproduce a general convergence result on stochastic approximation algorithms, as our own convergence result presented as Theorem~\ref{theo:convergence:msql} in the main paper heavily relies on it.

\subsubsection{Stochastic approximation}
Here we reproduce a convergence result for stochastic approximation\cite[Chapter 2]{borkar2008stochastic}. The result establishes conditions under which a stochastic approximation algorithm converges.
\setcounter{theorem}{-1}
\begin{theorem}\label{theo:borkar:sa}
Let us suppose that the following hold for the stochastic approximation setting of Sec.~\ref{sec:background}.
\begin{enumerate}
    \item The map $g: \mathbb{R}^k \to \mathbb{R}^k$ such that
    \begin{align*}
        g(\omega) = \mathbb{E}\left[\phi\left(z\right) \left( \tau(\omega) - g_{\omega}(x, a) \right) \right]
    \end{align*}
    is Lipschitz;
    \item The sequence of $m_{t}$ such that 
    \begin{align*}
        m_{t+1} = \phi(z_t)\left(\tau_{t+1} - g_{\omega_t}(z_t)\right) - g(\omega_t);
    \end{align*}
    is a martingale difference sequence with respect to $\{(\tau_s, \omega_s): s \leq t\}$ and is such that
    \begin{align*}
        \mathbb{E}\left[ \norm{m_{t+1} }^2 \mid \{(\tau_s, \omega_s): s \leq t\}  \right] \leq c_m \left( 1 + \norm{ \omega_t }^2 \right).
    \end{align*}
    \item The o.d.e
    \begin{align*}
        \dot{\omega} = g(\omega)
    \end{align*}
    has a unique and globally asymptotically stable equilibrium $\omega^*$ such that
    \begin{align*}
        \omega^* = \mathbb{E}\left[ \phi\left(z\right) \phi^T\left( z \right) \right]^{-1} \mathbb{E}\left[ \phi\left(z\right)\tau(\omega)  \right];
    \end{align*}
    \item The map $g_c : \mathbb{R}^k \to \mathbb{R}^k$ such that 
    \begin{align*}
        g_c(\omega) = \frac{g(c \omega)}{c}
    \end{align*}
    is such that $\lim_{c \to \infty} \norm{ g_c(\omega) - g_\infty(\omega) }$ uniformly on compacts for a $g_\infty: \mathbb{R}^k \to \mathbb{R}^k$ and that the origin is the unique and globally asymptotically stable equilibrium of
    \begin{align*}
        \dot{\omega} = g_\infty(\omega).
    \end{align*}
\end{enumerate}
Then, the sequence of stochastic approximation updated $\omega_t$ converges to $\omega^*$.
\end{theorem}
The result we reproduced guarantees that if, on expectation, the stochastic approximation update is smooth, the noise is well-behaved, the underlying dynamical system is asymptotically stable and the iterates remain bounded, the stochastic approximation update converges to the solution of the underlying dynamical system.

The proof of the convergence of the proposed multi $Q$-learning appears later, on Section~\ref{sec:app:mql}, and heavily resorts on Theorem~\ref{theo:borkar:sa}. In the next section, we prove the results that appeared in Section~\ref{sec:mbo} of the main document. The results are useful properties of the multi-Bellman operator.
%
\setcounter{subsection}{2}
\subsection{Multi-Bellman operator}\label{sec:app:mbo}
In this section, we reproduce the theoretical results established in Section~\ref{sec:mbo} and present their proofs. The results refer to useful properties of the multi-Bellman operator.

We start with Lemma~\ref{prop:contraction-gamma-n} of the main paper, which establishes that the multi-Bellman operator with parameter $n$ is contractive and that the contraction improves exponentially as $n$ increases.
\setcounter{lemma}{0}
\begin{lemma}\label{lma:contract_inf}
    The operator $\mathbf{H}^n$ is a contraction in the $\infty$-norm with contraction factor $\gamma^n$.\end{lemma}
\begin{proof}
We establish the result by induction. First, we show that $\mathbf{H}$ contracts in the $\infty$-norm.
We start by writing the expression $\norm{ \mathbf{H}q - \mathbf{H}p }_\infty$ in simpler terms, where $q$ and $p$ are value functions.
\begin{align*}
     \norm{ \mathbf{H}q - \mathbf{H}p }_\infty &= \max_{x, a} \left| \mathbb{E}\left[ r(x, a) + \gamma \max_{a' \in \mathcal{A}} q\left(x', a'\right) \right] - \mathbb{E}\left[ r(x, a) + \gamma \max_{a' \in \mathcal{A}} p\left(x', a'\right) \right] \right| \\
     &= \gamma \max_{x, a} \left| \mathbb{E}\left[ \max_{a' \in \mathcal{A}} q\left(x', a'\right) - \max_{a' \in \mathcal{A}} p\left(x', a'\right) \right] \right|.
\end{align*}
Next, we make use of Jensen's inequality to establish that
\begin{align*}
     \norm{\mathbf{H}q - \mathbf{H}p }_\infty&\leq \gamma \max_{x, a} \mathbb{E}\left[ \left| \max_{a' \in \mathcal{A}} q\left(x', a'\right) - \max_{a' \in \mathcal{A}} p\left(x', a'\right) \right| \right].
\end{align*}
We use that the absolute difference of maxima is less or equal than the maxima of absolute difference.
\begin{align*}
     \norm{ \mathbf{H}q - \mathbf{H}p }_\infty &\leq \gamma \max_{x, a} \mathbb{E}\left[ \max_{a' \in \mathcal{A}} \left| q\left(x', a'\right) - p\left(x', a'\right) \right| \right].
\end{align*}
Then, we use that the expectation of a random function is less or equal than its maximum.
\begin{align*}
     \norm{ \mathbf{H}q - \mathbf{H}p }_\infty &\leq \gamma \max_{x' \in \mathcal{X}} \max_{a' \in \mathcal{A}} \left| q\left(x', a'\right) - p\left(x', a'\right) \right| \\
     &= \gamma \max_{x', a'}   \left| q\left(x', a'\right) - p\left(x', a'\right) \right| \\
     &= \gamma \norm{q-p}_\infty
\end{align*}
Now, we make the induction step for $\mathbf{H}^{n+1} = \mathbf{H}\mathbf{H}^n$, observing that
\begin{align*}
     \norm{ \mathbf{H}\left(\mathbf{H}^n q\right) - \mathbf{H}\left(\mathbf{H}^n p\right) }_\infty &\leq \gamma \norm{ \mathbf{H}^n q - \mathbf{H}^n p }_\infty \\
     & = \gamma \gamma^n \norm{ q - p }_\infty \\
     & = \gamma^{n+1} \norm{ q - p }_\infty.
\end{align*}
We conclude the proof.
\end{proof}
%
%
Now, we prove Proposition~\ref{prop:contraction:mbo}, stating the contractiveness of the projected multi-Bellman operator.
\setcounter{proposition}{0}
\begin{proposition}
    There exists $N \in \mathbb{N}$ such that, for all $n \geq N$, 
     %
     %
     $\proj \mathbf{H}^n$ is a contraction in the $\mu$-norm.
\end{proposition}

\begin{proof}
We want to show that there exists $\lambda: \mathbb{N} \to \mathbb{R}$ such that
\begin{align*}
    \norm{ \proj \left(\mathbf{H}^n q\right)  - \proj \left(\mathbf{H}^n p\right)} \leq \lambda(n) \norm{ q - p }
\end{align*}
and such that, for all $n$ greater than an existing $N$, $\lambda(n)$ is strictly smaller than 1.

We start by showing that $\proj$ satisfies
\begin{align*}
    \norm{ \proj q - \proj p} \leq \phi^2_{\mathrm{max}} \sigma_\mathrm{max} \norm{q-p}_\infty.
\end{align*}
To see that, we start by recalling that 
\begin{align*}
    \left(\proj q\right) (x, a) = \phi^T(x, a) \mathbb{E}\left[\phi(x, a) \phi^T(x, a)\right]^{-1} \mathbb{E}\left[ \phi(x, a) q(x, a) \right].
\end{align*}
Then, we have that 
\begin{align*}
    \norm{ \proj q - \proj p} = \sqrt{\mathbb{E}\left[\left(\phi^T(x, a) \mathbb{E}\left[\phi(x, a) \phi^T(x, a)\right]^{-1} \mathbb{E}\left[ \phi(x, a)\left( q(x, a) - p(x, a) \right)\right]\right)^2 \right]}.
\end{align*}
Now, we make use of a Cauchy-Schwarz inequality to say that
\begin{align*}
    \norm{ \proj q - \proj p} \leq \sqrt{\mathbb{E}\left[\norm{\phi^T(x, a) \mathbb{E}\left[\phi(x, a) \phi^T(x, a)\right]^{-1}}_2^2 \norm{\mathbb{E}\left[ \phi(x, a)\left( q(x, a) - p(x, a) \right)\right]}_2^2 \right]}
\end{align*}
and we can write
\begin{align*}
    \norm{ \proj q - \proj p} &\leq \sqrt{\mathbb{E}\left[\norm{\phi^T(x, a) \mathbb{E}\left[\phi(x, a) \phi^T(x, a)\right]^{-1}}_2^2 \right]} \sqrt{\norm{\mathbb{E}\left[ \phi(x, a) \left(q(x, a) - p(x, a) \right)\right]}_2^2} \\
    &= \sqrt{\mathbb{E}\left[\norm{\phi^T(x, a) \mathbb{E}\left[\phi(x, a) \phi^T(x, a)\right]^{-1}}_2^2 \right]} \norm{\mathbb{E}\left[ \phi(x, a)\left( q(x, a) - p(x, a) \right)\right]}_2
\end{align*}
Now, we use the definition of the matrix norm induced by the euclidean norm in $\mathbb{R}^k$ for the first term and Jensen's inequality. We obtain that
\begin{align*}
    \norm{ \proj q - \proj p} &\leq \sqrt{\mathbb{E}\left[\norm{\phi(x, a)}_2^2 \norm{\mathbb{E}\left[\phi(x, a) \phi^T(x, a)\right]^{-1}}_2^2 \right]}\mathbb{E}\left[\norm{ \phi(x, a) \left(q(x, a) - p(x, a)\right)}_2 \right] \\
    &\leq \sqrt{ \phi^2_\mathrm{max} \sigma_\mathrm{max}^2} \cdot \mathbb{E}\left[\norm{ \phi(x, a) }_2 \left| q(x, a) - p(x, a)\right| \right] \\
    &\leq \phi_\mathrm{max}^2 \sigma_\mathrm{max} \norm{q - p}_\infty.
\end{align*}

To conclude the proof, we start by combining the inequality we just established and Lemma~\ref{lma:contract_inf}.
\begin{align*}
    \norm{ \proj \left(\mathbf{H}^n q \right) - \proj \left(\mathbf{H}^n p \right)} &\leq \phi^2_\mathrm{max} \sigma_\mathrm{max}  \norm{ \mathbf{H}^n q - \mathbf{H}^n p }_\infty \\
    & \leq \phi^2_\mathrm{max} \sigma_\mathrm{max} \gamma^n  \norm{ q - p }_\infty
\end{align*}
Let us consider $\mu_\mathrm{min}$ the probability of the least likely state-action pair according to distribution $\mu$, assuming all state-action pairs have a non-zero probability of being visited---that is bigger or equal than $\mu_\mathrm{min}$ (we are defining the constant in revised Assumption~\ref{ass:data:iid} of the revised paper). Finally, we use that $\norm{q}_\infty \leq \frac{1}{\mu_\mathrm{min}}\norm{q}$ and finish the proof by using $\lambda(n) = \frac{\phi^2_\mathrm{max} \sigma_\mathrm{max}}{\mu_\mathrm{min}}\gamma^n$. It is true that $\lambda(n) < 1$ for all $n \geq N$ when we consider $N = \lceil - \log_\gamma \left( \frac{\phi_\mathrm{max}^2\sigma_\mathrm{max}}{\mu_\mathrm{min}} \right) \rceil $.\footnote{$\lceil \cdot \rceil : \mathbb{R}^+ \to \mathbb{N}$ is the ceiling function, giving the smallest natural that is larger or equal to its argument.}
\end{proof}
\begin{proposition}
    For all $n\geq N$, where $N$ is identified in Corollary~\ref{prop:contraction:mbo},  $\tilde{\omega}^n$ is such that
    \begin{align}
        \norm{ q^* - q_{\tilde{\omega}^n} } \leq \frac{1}{1 - \frac{\sigma_\mathrm{max}\phi_\mathrm{max}^2}{\mu_\mathrm{min}}\gamma^n}\norm{ q^* - q_{\omega^*} },
    \end{align}
where $\sigma_\mathrm{max}=\norm{\mathbb{E} \left[ \phi(x, a) \phi^T(x, a) \right]^{-1}}_2$ and $\phi_\mathrm{max}=\max_{x, a}\norm{\phi(x, a)}_2$.
\end{proposition}
In Proposition~\ref{cor:qual} of the main text submitted, the constant $\mu_\mathrm{min}$ is missing. We are adding it in the revised version of the manuscript.

\begin{proof}
Let us use the triangle inequality with vertex $q_{\omega^*}$.
\begin{align*}
    \norm{ q^* - q_{\tilde{\omega}^n} } \leq \norm{ q^* - q_{\omega^*} } + \norm{ q_{\omega^*} - q_{\tilde{\omega}^n} }.
\end{align*}
Now we focus on the second term of the right-hand side above. We have that
\begin{align*}
    \norm{ q_{\omega^*} - q_{\tilde{\omega}^n} } &= \norm{ \proj \left( \mathbf{H}^n q^* \right) - \proj\left( \mathbf{H}^n q_{\tilde{\omega}^n} \right) }.
\end{align*}
From Proposition~\ref{prop:contraction:mbo}, we have that
\begin{align*}
    \norm{ q_{\omega^*} - q_{\tilde{\omega}^n} } &\leq  \frac{\sigma_\mathrm{max} \phi_\mathrm{max}^2}{\mu_\mathrm{min}}\gamma^n \norm{  q^* -  q_{\tilde{\omega}^n} }.
\end{align*}
We can conclude the result taking algebraic operations.
\end{proof}

%
\subsection{Multi $Q$-learning}\label{sec:app:mql}

\setcounter{theorem}{0}
In this section, we prove Theorem~\ref{theo:convergence:msql}, that we reproduce below for convenience.
\begin{theorem}
    There exists $N$ such that, for all $n$ larger than $N$, the sequence of $\omega_t^n$ is such that
    \begin{align*}
        \lim_{t \to \infty} \norm{\tilde{\omega}^n - \omega_t^n}_2 = 0.
    \end{align*}
\end{theorem}
To establish the result, we verify that we are in the conditions where Theorem~\ref{theo:borkar:sa} holds.
We do so by proving the following lemmas, one for each of the asumptions of the theorem.
\begin{lemma}\label{lemma:lip}
    The map $g: \mathbb{R}^k \to \mathbb{R}^k$ such that
    \begin{align*}
        g(\omega) = \mathbb{E} \left[ \phi(x, a) \left( \tau^n(\omega) - q_{\omega^n}(x, a) \right) \right]
    \end{align*}
    is Lipschitz.
\end{lemma}
\begin{proof}
    We start by writing
    \begin{align*}
        \norm{ g( \omega ) - g( \theta ) }_2 &= \norm{ \mathbb{E} \left[ \phi(x, a) \left( \tau^n(\omega) - q_{\omega}(x, a) \right) \right] - \mathbb{E} \left[ \phi(x, a) \left( \tau^n(\theta) - q_{\theta}(x, a) \right) \right] }_2 \\
        &\leq \norm{ \mathbb{E} \left[ \phi(x, a) \left(\tau^n(\omega) - \tau^n(\theta)\right) \right] }_2 + \norm{ \mathbb{E} \left[ \phi(x, a) \left(q_{\theta}(x, a) - q_{\omega}(x, a) \right) \right] }_2.
    \end{align*}
    For the first term, we can observe that it equals.
    \begin{align*}
        \norm{\mathbb{E}\left[ \phi(x, a) \phi^T(x, a) \right]\left( \proj \left( \mathbf{H}^n q_\omega \right) - \proj \left(\mathbf{H}^n q_\theta \right)\right)}_2
    \end{align*}
    Using Jensen's and Cauchy-Schwarz inequality and Proposition~\ref{prop:contraction-gamma-n}, we conclude that
    \begin{align*}
        \norm{ \mathbb{E} \left[ \phi(x, a) \left(\tau^n(\omega) - \tau^n(\theta)\right) \right] }_2 &\leq  \mathbb{E}\left[ \norm{\phi(x, a) }_2 \left|\tau^n(\omega) - \tau^n(\theta)\right| \right] \\ 
        &\leq \sigma_\mathrm{max}\phi^2_\mathrm{max} \gamma^n \norm{ q_\omega - q_\theta }_\infty \\
        &\leq \frac{\sigma_\mathrm{max}\phi^3_\mathrm{max}}{\mu_\mathrm{min}} \gamma^n \norm{ \omega - \theta}_2
    \end{align*}
    Now we can take care of the second term by making use of Jensen's and Cauchy-Schwarz' inequality.
    \begin{align*}
        \norm{ \mathbb{E} \left[ \phi(x, a) \left(q_{\theta}(x, a) - q_{\omega}(x, a) \right) \right] }_2 &\leq  \mathbb{E} \left[ \norm{ \phi(x, a) }_2 \left| q_{\theta}(x, a) - q_{\omega}(x, a) \right| \right]\\
        &\leq \phi_\mathrm{max} \norm{ q_{\theta} - q_{\omega}}_\infty \\
        &\leq  \frac{\phi_\mathrm{max}^2}{\mu_\mathrm{min}} \norm{ \omega - \theta}_2.
    \end{align*}
    Finally, by considering the sum of the coefficients from the first and second terms,  we have that
    \begin{align*}
        \norm{ g( \omega ) - g( \theta ) }_2 &\leq \frac{\sigma \phi^3_\mathrm{max} \gamma^n + \phi^2_\mathrm{max}}{\mu_\mathrm{min}} \norm{ \omega - \theta}_2,
    \end{align*}
    or, equivalently, that $g$ is Lipschitz continuous.
\end{proof}
\begin{lemma}\label{lemma:martingale}
    The sequence of $m_{t}$ such that 
    \begin{align*}
        m_{t+1} = \phi(x_t, a_t)\left(\tau^n_{t+1} - g_{\omega_t}(x_t, a_t)\right) - g(\omega_t)
    \end{align*}
    is a martingale difference sequence with respect to $\{(\tau^n_s, \omega_s): s \leq t\}$ and is such that
    \begin{align*}
        \mathbb{E}\left[ \norm{ m_{t+1} }^2 \mid \{(\tau^n_s, \omega_s): s \leq t\}  \right] \leq c_m \left( 1 + \norm{ \omega_t }^2 \right).
    \end{align*}
\end{lemma}
\begin{proof}
A martingale difference sequence has zero expectation conditioned on the past. Such condition is evident once we consider Assumption~\ref{ass:data:iid}. Then, a martingale difference sequence also has finite first moment. That becomes evident when we observe that every term on the definition is bounded. Finally, since, again, every term on the definition is also bounded, the second moment is bounded.
\end{proof}
\begin{lemma}\label{lemma:ode}
The o.d.e
    \begin{align*}
        \dot{\omega} = g(\omega)
    \end{align*}
    has a unique and globally asymptotically stable equilibrium $\omega^*$ such that
    \begin{align*}
        \omega^* = \mathbb{E}\left[ \phi\left(x, a\right) \phi^T\left( x, a \right) \right]^{-1} \mathbb{E}\left[ \phi\left(x, a\right)\tau^n(\omega^*)  \right].
    \end{align*}
\end{lemma}
\begin{proof}
We already know that the equilibrium $\omega^*$ exists and is unique.
Let us consider a Lyapunov function $l: \mathbb{R}^k  \to \mathbb{R}$ such that $l(\omega) = \frac{1}{2}\norm{\omega^* - \omega}^2$. Existence of $\omega^*$ is guaranteed from Banach's fixed-point theorem and Proposition~\ref{prop:contraction:mbo}. It is clear that $l(\omega) = 0$ if $\omega = \omega^*$ and that $l(\omega) > 0$ if $\omega \neq \omega^*$. Now, we show that $\dot{l}(\omega) < 0$ if $\omega \neq \omega^*$.

We start by noticing that $\dot{l}(\omega) = \nabla l \cdot \dot{\omega}$. Since $\nabla l = - \left( \omega^* - \omega \right) $ and $\dot{\omega} = g(\omega)$, we have that
\begin{align*}
    \dot{l}(\omega) = - \left( \omega^* - \omega \right) g(\omega).
\end{align*}
Since $g(\omega^*) = 0$, it is also true that
\begin{align*}
    \dot{l}(\omega) &= - \left( \omega^* - \omega \right) \left(g\left(\omega\right) - g\left(\omega^*\right) \right)
\end{align*}
For now, let us focus on the second term of the right hand side and write
\begin{align*}
    g(\omega) - g(\omega^*) &= \mathbb{E} \left[ \phi(x, a) \left( \tau^n(\omega) - \tau^n(\omega^*) \right) \right] - \mathbb{E} \left[ \phi(x, a) \left( q_{\omega^*}(x, a) - q_{\omega}(x, a) \right) \right]
\end{align*}
Let us make two assertions, one for each of the terms on the right-hand side of the equation. First,
\begin{align*}
    \norm{\mathbb{E} \left[ \phi(x, a) \left( \tau^n(\omega) - \tau^n(\omega^*) \right) \right]}_2 = \norm{\mathbb{E} \left[ \phi(x, a) \left( \tau^n(\omega^*) - \tau^n(\omega) \right) \right]}_2.
\end{align*}
Using Jensen's and Cauchy-Schwarz' inequality, we get that
\begin{align*}
    \norm{\mathbb{E} \left[ \phi(x, a) \left( \tau^n(\omega) - \tau^n(\omega^*) \right) \right]}_2 &\leq \mathbb{E}\left[ \norm{\phi(x, a)}_2 \left| \tau^n(\omega^*) - \tau^n(\omega) \right| \right]
\end{align*}
and, using the fact that $\left| \max_{z} f(z) - \max_{z} g(z) \right| \leq \max_{z} \left| f(z) - g(z) \right|$, consequently that
\begin{align*}
    \norm{\mathbb{E} \left[ \phi(x, a) \left( \tau^n(\omega) - \tau^n(\omega^*) \right) \right]}_2 &\leq \phi_\mathrm{max} \gamma^n \norm{q_{\omega^*} - q_\omega}_\infty \\
    &\leq \frac{\phi_\mathrm{max}^2}{\mu_\mathrm{min}} \gamma^n \norm{\omega^* - \omega}
\end{align*}
Second, we have that
\begin{align*}
    \mathbb{E} \left[ \phi(x, a) \left( q_{\omega^*}(x, a) - q_{\omega}(x, a) \right) \right] = \mathbb{E}\left[ \phi(x, a) \phi^T(x, a) \right] \left( \omega^* - \omega \right).
\end{align*}
Having established the two assertions mentioned, we make use of them in the following.
\begin{align*}
    \dot{l}(\omega) &= - (\omega^* - \omega) \mathbb{E} \left[ \phi(x, a) \left( \tau^n(\omega) - \tau^n(\omega^*) \right) \right] -\\
    & \qquad \qquad - (\omega^*  - \omega) \mathbb{E}\left[ \phi(x, a) \phi^T(x, a) \right] (\omega^* - \omega) \\
    &\leq \norm{\left( \omega^* - \omega \right) \mathbb{E} \left[ \phi(x, a) \left( \tau^n(\omega) - \tau^n(\omega^*) \right) \right]}_2 -\\
    & \qquad \qquad - (\omega^*  - \omega) \mathbb{E}\left[ \phi(x, a) \phi^T(x, a) \right] (\omega^* - \omega) \\
    & \leq \frac{\phi_\mathrm{max}^2}{\mu_\mathrm{min}} \gamma^n \norm{\omega^* - \omega}^2_2 - \lambda_{\mathrm{max}} \norm{\omega^* - \omega}^2_2 \\
    &= \left( \frac{\phi_\mathrm{max}^2}{\mu_\mathrm{min}} \gamma^n - \lambda_\mathrm{max}\right) \norm{\omega^* - \omega}_2^2,
\end{align*}
where $\lambda_{\mathrm{max}}$ is the largest eigenvalue of the covariance matrix $\mathbb{E}\left[ \phi(x, a) \phi^T(x, a) \right]$.
The expression is negative if n is large enough. That concludes the proof.

\end{proof}
%
%
\begin{lemma}\label{lemma:bound}
    The map $g_c : \mathbb{R}^k \to \mathbb{R}^k$ such that 
    \begin{align*}
        g_c(\omega) = \frac{g(c \omega)}{c}
    \end{align*}
    is such that $\lim_{c \to \infty} \norm{ g_c(\omega) - g_\infty(\omega) } = 0$ uniformly on compacts for some $h_\infty: \mathbb{R}^k \to \mathbb{R}^k$ and that the origin is the unique and globally asymptotically stable equilibrium of
    \begin{align*}
        \dot{\omega} = g_\infty(\omega).
    \end{align*}
\end{lemma}
\begin{proof}
Let us expand the definition and obtain
\begin{align*}
    g_c(\omega) &= \frac{\mathbb{E} \left[ \phi(x, a) \left( \tau^n - q_{c \omega^n}(x, a) \right) \right]}{c} \\
    &= \frac{\mathbb{E} \left[ \phi(x, a) \tau^n \right]}{c} - \mathbb{E}\left[\phi(x, a)q_{\omega}(x, a)\right].
\end{align*}
As $c\to \infty$, we have that $g_c(\omega) \to g_\infty(\omega)$ uniformly, using $g_\infty(\omega) = -\mathbb{E}\left[ \phi(x, a) q_{\omega}(x, a) \right]$.

We then have that the ordinary differential equation
\begin{align*}
    \dot{\omega} &= g_\infty(\omega) \\
    &= -\mathbb{E}\left[ \phi(x, a) q_{\omega}(x, a) \right] \\
    &= - \mathbb{E}\left[ \phi(x, a) \phi^T(x, a) \right] \omega
\end{align*}
is linear and time-invariant. Since the matrix above is invertible, we have that the origin is its unique and globally asymptotically stable equilibrium. We conclude the result.
\end{proof}

\end{document}